\documentclass[12pt,a4paper]{article}
\usepackage{bbm}
\usepackage{amsmath,amssymb}
\usepackage{setspace}
\usepackage{algpseudocode,algorithm}
\usepackage{tabularx}
\usepackage{amsthm}
\usepackage{graphicx}
\usepackage{here}
\usepackage{url}
\usepackage{txfonts}

\usepackage{natbib}
\bibpunct[:]{(}{)}{;}{a}{}{;}
\usepackage[usenames]{color}
\usepackage{authblk}

\allowdisplaybreaks[1]

\providecommand{\keywords}[1]{\textbf{\textit{Keywords---}} #1}

\newtheorem{theorem}{Theorem}[section]
\newtheorem{corollary}{Corollary}[section]

\newcommand{\argmin}{\mathop{\rm argmin}\limits}

\makeatletter
  
  \@addtoreset{equation}{section}
\makeatother

\title{Robust and Sparse Regression in GLM by Stochastic Optimization}
\date{} 

\author[1]{Takayuki Kawashima} 
\author[1,2,3]{Hironori Fujisawa} 

\affil[1]{ Department of Statistical Science, The Graduate University for Advanced Studies, Tokyo  \authorcr  E-mail: \texttt{t-kawa@ism.ac.jp}}

\affil[2]{The Institute of Statistical Mathematics, Tokyo  \authorcr E-mail: \texttt{fujisawa@ism.ac.jp} }

\affil[3]{ Department of Mathematical Statistics, Nagoya University Graduate School of Medicine  }

\begin{document}
\maketitle

\begin{abstract}
The generalized linear model (GLM) plays a key role in regression analyses. 
In high-dimensional data, the sparse GLM has been used but it is not robust against outliers. 
Recently, the robust methods have been proposed for the specific example of the sparse GLM. 
Among them, we focus on the robust and sparse linear regression based on the $\gamma$-divergence.
The estimator of the $\gamma$-divergence has strong robustness under heavy contamination. 
In this paper, we extend the robust and sparse linear regression based on the $\gamma$-divergence to the robust and sparse GLM based on the $\gamma$-divergence with a stochastic optimization approach in order to obtain the estimate. 
We adopt the randomized stochastic projected gradient descent as a stochastic optimization approach and extend the established convergence property to the classical first-order necessary condition. 
By virtue of the stochastic optimization approach, we can efficiently estimate parameters for very large problems. 
Particularly, we show the linear regression, logistic regression and Poisson regression with $L_1$ regularization in detail as specific examples of robust and sparse GLM. 
%
%
In numerical experiments and real data analysis, the proposed method outperformed comparative methods.
\end{abstract}


\section{Introduction}\label{Sect: intro}
The regression analysis is a fundamental tool in data analysis. 
The Generalized Linear Model (GLM) \citep{10.2307/2344614, mccullagh1989generalized} is often used and includes many important regression models, including linear regression, logistic regression and Poisson regression. 
Recently, the sparse modeling has been popular in GLM to treat high-dimensional data and, for some specific examples of GLM, the robust methods have also been incorporated (linear regression: \citet{RePEc:bes:jnlasa:v:102:y:2007:m:december:p:1289-1299, alfons2013}, logistic regression: \citet{doi:10.1093/bioinformatics/btt078, doi:10.1080/10618600.2012.737296} ). 

\citet{e19110608} proposed a robust and sparse regression based on the $\gamma$-divergence \citep{Fujisawa:2008:RPE:1434999.1435056}, which has a strong robustness that the latent bias can be sufficiently small even under heavy contamination. 
The proposed method showed better performances than the past methods by virtue of strong robustness. 
A coordinate descent algorithm with Majorization-Minimization algorithm was constructed as an efficient estimation procedure for linear regression, but it is not always useful for GLM. 
To overcome this problem, we propose a new estmation procedure with a stochastic optimization approach, which largely reduces the computational cost and is easily applicable to any example of GLM. 
In many stochastic optimization approaches, we adopt the randomized stochastic projected gradient descent (RSPG) proposed by \citet{Ghadimi:2016}. 
In particular, when we consider the Poisson regression with $\gamma$-divergence, although the loss function includes a hypergeometric series and demands high computational cost, the stochastic optimization approach can easily overcome this difficulty. 

In Section~\ref{2}, we review the robust and sparse regression via $\gamma$-divergence. 
In Section~\ref{3}, the RSPG is explained with regularized expected. 
In Section~\ref{4}, an online algorithm is proposed for GLM and the robustness of online algorithm is described with some typical examples of GLM. 
In Section~\ref{Sect: convergence prop}, the convergence property of the RSPG is extended to the classical first-order necessary condition. 
In Sections~\ref{6}~and~\ref{7}, numerical experiments and real data analysis are illustrated to show better performances than comparative methods. 
Concluding remarks are given in Section~\ref{8}.  
%
%
%
%
%
%
%
%
%

\section{Regression via $\gamma$-divergence}\label{2}
\subsection{Regularized Empirical risk minimization}
We suppose $g$ is the underlying probability density function and $f$ is a parametric probability density function. 
Let us define the $\gamma$-cross entropy for regression given by
\begin{align*}
&d_{\gamma} (g(y|x),f(y|x);g(x))  \\
&=  -\frac{1}{\gamma} \log \int \frac{ \int g(y|x) f(y|x)^{\gamma} dy }{  \left( \int f(y|x)^{1+\gamma}dy\right)^\frac{\gamma}{1+\gamma}} g(x)dx \\
& = -\frac{1}{\gamma} \log \int \int  \frac{   f(y|x)^{\gamma}  }{  \left( \int f(y|x)^{1+\gamma}dy\right)^\frac{\gamma}{1+\gamma}} g(x,y)dxdy, \\
& = -\frac{1}{\gamma} \log E_{g(x,y)} \left[  \frac{   f(y|x)^{\gamma}  }{  \left( \int f(y|x)^{1+\gamma}dy\right)^\frac{\gamma}{1+\gamma}}  \right]. 
\end{align*}
The $\gamma$-divergence for regression is defined by 
\begin{align*}
&D_{\gamma} (g(y|x),f(y|x);g(x))  \\
&=-d_{\gamma} (g(y|x),g(y|x);g(x)) +d_{\gamma} (g(y|x),f(y|x);g(x)).
\end{align*}
The main idea of robustness in the $\gamma$-divergence is based on density power weight $f(y|x)^{\gamma}$ which gives a small weight to the terms related to outliers. 
Then, the parameter estimation using the $\gamma$-divergence becomes robust against outliers and it is known for having a strong robustness, which implies that the latent bias can be sufficiently small even under heavy contamination.
More details about robust properties were investigated by \citet{Fujisawa:2008:RPE:1434999.1435056}, \citet{RePEc:oup:biomet:v:102:y:2015:i:3:p:559-572.} and \citet{e19110608}. 

Let $f(y|x;\theta)$ be the parametric probability density function with parameter $\theta$. 
The target parameter can be considered by 
\begin{align*}
\theta^*_{\gamma} & = \argmin_{\theta} D_{\gamma}(g(y|x),f(y|x;\theta);g(x)) \nonumber \\
&= \argmin_{\theta} d_{\gamma}(g(y|x),f(y|x;\theta);g(x)).
\end{align*}
Moreover, we can also consider the target parameter with a convex regularization term, given by
\begin{align}\label{def sparse exp est}
\theta^*_{\gamma, pen} & = \argmin_{\theta} D_{\gamma}(g(y|x),f(y|x;\theta);g(x)) + \lambda P(\theta) \nonumber \\
&= \argmin_{\theta} d_{\gamma}(g(y|x),f(y|x;\theta);g(x))+ \lambda P(\theta) ,
\end{align}
where $P(\theta)$ is a convex regularization term for parameter $\theta$ and $\lambda$ is a tuning parameter. 
As an example of convex regularization term, we can consider $L_1$ (Lasso, \citealt {Tibshirani94regressionshrinkage}), elasticnet \citep{Zou05regularizationand}, the indicator function of a closed convex set \citep{Kivinen95exponentiatedgradient, Duchi:2008:EPL:1390156.1390191} and so on.
In what follows, we refer to the regression based on the $\gamma$-divergence as the $\gamma$-regression.

Let $(x_1,y_1) ,\ldots, (x_n,y_n)$ be the observations randomly drawn from the underlying distribution $g(x,y)$. 
The $\gamma$-cross entropy can be empirically estimated by 
\begin{align*}
&\bar{d}_\gamma(f(y|x;\theta)) = -\frac{1}{\gamma} \log \frac{1}{n} \sum_{i=1}^n  \frac{   f(y_i|x_i)^{\gamma}  }{  \left( \int f(y|x_i)^{1+\gamma}dy\right)^\frac{\gamma}{1+\gamma}} .
\end{align*}
By virtue of (\ref{def sparse exp est}), the sparse $\gamma$-estimator can be proposed by 
\begin{align}\label{def sparse emp est}
\hat{ \theta}_{\gamma, pen}&=\argmin_{\theta} \bar{d}_\gamma(f(y|x;\theta))+ \lambda P(\theta) .
\end{align}
To obtain the minimizer, we solve a non-convex and non-smooth optimization problem. 
%
%
Iterative estimation algorithms for such a problem can not easily achieve numerical stability and efficiency. 

\subsection{MM algorithm for $\gamma$-regression}
\citet{e19110608} proposed the iterative estimation algorithm for (\ref{def sparse emp est}) by Majorization-Minimization algorithm (MM algorithm) \citep{hunter:mm}. 
It has a monotone decreasing property, i.e., the objective function monotonically decreases at each iterative step, which property leads to numerical stability and efficiency. 
In particular, the linear regression with $L_1$ penalty was deeply considered. 

Here, we explain the idea of MM algorithm briefly.
Let $h(\eta)$ be the objective function. 
Let us prepare the majorization function $h_{MM}$ satisfying
\begin{align*}
h_{MM}(\eta^{(m)}|\eta^{(m)})& = h(\eta^{(m)}), \\
h_{MM}(\eta|\eta^{(m)}) &\geq h(\eta) \ \ \mbox{ for all }  \eta,
\end{align*}
where $\eta^{(m)}$ is the parameter of the $m$-th iterative step for $m=0,1,2,\ldots$. 
MM algorithm optimizes the majorization function instead of the objective function as follows:
\begin{align*}
\eta^{(m+1)} = \argmin_{\eta} h_{MM}(\eta|\eta^{(m)}).
\end{align*}
Then, we can show that the objective function $h(\eta)$ monotonically decreases at each iterative step, because
\begin{align*}
h(\eta^{(m)}) &= h_{MM}(\eta^{(m)}|\eta^{(m)}) \\
& \geq h_{MM}(\eta^{(m+1)}|\eta^{(m)}) \\
& \geq h(\eta^{(m+1)}).
\end{align*}
Note that $\eta^{(m+1)}$ is not necessary to be the  minimizer of $h_{MM}(\eta|\eta^{(m)})$. 
We only need 
\begin{align*}
h_{MM}(\eta^{(m)}|\eta^{(m)}) \geq h_{MM}(\eta^{(m+1)}|\eta^{(m)}).
\end{align*} 
The problem on MM algorithm is how to make a majoraization function $h_{MM}$. 

In \citet{e19110608}, the following majorization function was proposed by using Jensen's inequality:
\begin{align}\label{MM func gam}
h_{MM}(\theta|\theta^{(m)}) = - \frac{1}{\gamma} \sum_{i=1}^n \alpha^{(m)}_{i} \log \left\{ \frac{  f(y_i|x_i;\theta)^{\gamma} }{ \left( \int f(y|x_i;\theta)^{1+\gamma}dy\right)^{\frac{\gamma}{1+\gamma}} } \right\} + \lambda P(\theta) ,
\end{align}
where 
\begin{align*}
\alpha^{(m)}_{i}= \frac{  \frac{  f(y_i|x_i;\theta^{(m)} )^{\gamma} }{ \left( \int f(y|x_i;\theta^{(m)})^{1+\gamma}dy\right)^{\frac{\gamma}{1+\gamma}} }   }{   \sum_{l=1}^n \frac{  f(y_l|x_l;\theta^{(m)} )^{\gamma} }{ \left( \int f(y|x_l;\theta^{(m)})^{1+\gamma}dy\right)^{\frac{\gamma}{1+\gamma}} } }. 
\end{align*}
Moreover, for linear regression $y=\beta_0 + x^T \beta + e \ (e \sim N(0,\sigma^2) )$ with $L_1$ regularization, the following majorization function and iterative estimation algorithm based on a coordinate descent method were obtained:
\begin{align*}
 h_{MM, \ linear}(\theta|\theta^{(m)}) &= \frac{1}{2(1+\gamma)}  \log \sigma^{2} + \frac{1}{2} \sum_{i=1}^n \alpha^{(m)}_i \frac{(y_i -\beta_0 - x_{i}^{T}\beta )^{2}}{\sigma^{2}} +\lambda ||\beta||_1 , \\
\beta_0^{(m+1)} &=  \sum_{i=1}^n \alpha_i^{(m)} (y_i-{x_i}^T \beta^{(m)}) , \\
\beta_{j}^{(m+1)}  &= \frac{ S\left(\sum_{i=1}^n \alpha^{(m)}_{i}(y_i-\beta^{(m+1)}_0 - r_{i,-j}^{(m)} )x_{i j} , \  {\sigma^2}^{(m)}\lambda \right) }{ \left( \sum_{i=1}^n \alpha^{(m)}_i x^{2}_{ij} \right) } \ \ (j=1,\ldots,p), \\
{\sigma^{2}}^{(m+1)} &=  (1 + \gamma) \sum_{i=1}^n \alpha^{(m)}_i ( y_i  - \beta^{(m+1)}_0 - x^{T}_{i} \beta^{(m+1)} )^{2} ,
\end{align*}
where $S(t,\lambda)=\textrm{sign}(t)(|t|-\lambda) $ and $r_{i,-j}^{(m)}=\sum_{ k \neq  j} x_{ik} ( \mathbbm{1}_{( k <  j) }  \beta_k^{(m+1)} +  \mathbbm{1}_{ (k > j) } \beta_k^{(m)} )$.

\subsection{Sparse $\gamma$-Poisson regression case}
Typical GLMs are a linear regression, logistic regression and Poisson regression: The former two regressions are easily treated with the above coordinate descent algorithm, but the Poisson regression has a problem as described in the following
Here, we consider a Poisson regression with a regularization term. 
Let $f(y|x;\theta)$ be the conditional density with $\theta=(\beta_0,\beta)$, given by 
\begin{align*}
f(y|x;\theta) =  \frac{\exp(-\mu_{x}(\theta)) }{y!} \mu_{x}(\theta)^y,
\end{align*}
where $\mu_{x}(\theta) = \mu_{x}(\beta_0,\beta) = \exp(\beta_0+x^T\beta)$. 
By virtue of (\ref{MM func gam}), we can obtain the majorization function for Poisson regression with a regularization term, given by 
\begin{align}\label{major func of gamma poisson}
&h_{MM, \ poisson}(\theta|\theta^{(m)}) = - \sum_{i=1}^n \alpha_{i}^{(m)} \log \frac{\exp(-\mu_{x_i}(\theta)) }{y_{i}!} \mu_{x_i}(\theta)^{y_i} \nonumber  \\
 & \qquad +\frac{1}{1+\gamma} \sum_{i=1}^n \alpha_{i}^{(m)} \log \left\{  \sum_{y=0}^\infty \frac{ \exp( -(1+\gamma)\mu_{x_i}(\theta))}{y!^{1+\gamma}} \mu_{x_i}(\theta)^{(1+\gamma)y}  \right\} + \lambda P(\theta).
\end{align}
The second term contains the hypergeometric series, and then we can not obtain a closed form on the MM algorithm with respect to the parameters $\beta_0, \beta$ although this series converges (see Sect.~\ref{update algo:online gam poisson}). 
Therefore, we can not derive an efficient iterative estimation algorithm based on a coordinate descent method in a similar way to in \citet{e19110608}. 
Other sparse optimization methods which use a linear approximation on the loss function, e.g., proximal gradient descent \citep{RePEc:cor:louvco:2007076, Duchi:2009:EOB:1577069.1755882, Beck:2009:FIS:1658360.1658364}, can solve (\ref{major func of gamma poisson}).
However, these methods require at least sample size $n$ times of an approximate calculation for the hypergeometric series at each iterative step in sub-problem  $\argmin_{\theta} h_{MM} (\theta|\theta^{(m)})$. 
%
%
Therefore, it requires high computation cost, especially for very large problems.
We need another optimization approach to overcome such problems. 
In this paper, we consider minimizing the regularized expected risk (\ref{def sparse exp est}) directly by a  stochastic optimization approach. 
In what follows, we refer to the sparse $\gamma$-regression in GLM as the sparse $\gamma$-GLM. 

\section{Stochastic optimization approach for regularized expected risk minimization}\label{3}
The regularized expected risk minimization is generally the following form: 
\begin{align}\label{RERM def}
\Psi^{*} \coloneqq \min_{\theta \in \Theta} \left\{ \Psi(\theta) \coloneqq E_{(x,y)} \left[ l((x,y);\theta) \right] + \lambda P(\theta) \right\},
\end{align}
where $\Theta$ is a closed convex set in $\mathbb{R}^n$, $l$ is a loss function with a parameter $\theta$ and $\Psi(\theta)$ is bounded below over $\Theta$ by $\Psi^* > - \infty$. 
Stochastic optimization approach solves (\ref{RERM def}) sequentially. 
More specifically, we draw a sequence of i.i.d. paired samples $(x_1,y_1),(x_2,y_2),\ldots,(x_t,y_t),\ldots$ and, at $t$-th time, update the parameter $\theta^{(t)}$ based on the latest paired sample $(x_t,y_t)$ and the previous updated parameter $\theta^{(t-1)}$. 
Therefore, it requires low computational complexity per iteration and stochastic optimization can scale well for very large problems. 

\subsection{Stochastic gradient descent}\label{SGD sect}
The stochastic gradient descent (SGD) is one of popular stochastic optimization approaches and is widely used in machine learning community \citep{bottou2010large}. 
The SGD takes the form
\begin{align}\label{update form sgd}
\theta^{(t+1)} = \argmin_{\theta  \in \Theta} \left\langle \nabla l((x_t,y_t);\theta^{(t)})   , \theta \right\rangle + \lambda P(\theta)  +\frac{1}{2 \eta_t}\|\theta-\theta^{(t)} \|^{2}_2,
\end{align}
where $\eta_t$ is a step size parameter. 
For some important examples, e.g., $L_1$ regularization, (\ref{update form sgd}) can be solved in a closed form. 
%

When a loss function $l$ is convex (possibly non-differentiable) and $\eta_t$ is set to be appropriate, e.g., $\eta_t= {\cal O} \left( \frac{1}{ \sqrt{t}} \right) $, under some mild conditions, the convergence property was established for the average of the iterates, i.e., $\bar{\theta}_{T}=\frac{1}{T} \sum_{t=1}^T \theta^{(t)}$ as follows (see, e.g., \cite{Bubeck:2015:COA:2858997.2858998}): 
\begin{align*}
E \left[ \Psi(\bar{\theta}_T) \right] - \Psi^* \leq {\cal O} \left( \frac{1}{\sqrt{T}} \right),
\end{align*}
where the expectation is taken with respect to past paired samples $(x_t, y_t) \ldots (x_T,y_T)$. 
Moreover, for some variants of SGD, e.g., RDA \citep{Xiao2010DualAM}, Mirror descent \citep{DBLPDuchiSST10}, Adagrad \citep{journals/jmlr/DuchiHS11}, the convergence property was established under similar assumptions. 
%

These methods assume that a loss function is convex to establish the convergence property, but the loss function is non-convex in our problem (\ref{def sparse exp est}). 
Then, we can not adopt these methods directly. 
%
Recently, for non-convex loss function with convex regularization term, randomized stochastic projected gradient (RSPG) was proposed by \cite{Ghadimi:2016}. 
Under some mild conditions, the convergence property was established. 
Therefore, we consider applying the RSPG to our problem (\ref{def sparse exp est}). 

\subsection{Randomized stochastic projected gradient} 
%
First, we explain the RSPG, following \cite{Ghadimi:2016}. 
%
%
The RSPG takes the form 
\begin{align}\label{update form rspg}
 \theta^{(t+1)} = \argmin_{\theta  \in \Theta } \left\langle \frac{1}{ m_t} \sum_{i=1}^{m_t} \nabla l((x_{t,i},y_{t,i});\theta^{(t)})  , \theta \right\rangle + \lambda P(\theta)  +\frac{1}{ \eta_t} V( \theta , \theta^{(t)} )  ,
\end{align}
where $m_t$ is the size of mini-batch at $t$-th time, $(x_{t,i},y_{t,i})$ is the $i$-th mini-batch sample at $t$-th time and 
\begin{align*}
V(a,b) = w(a) -w(b) - \langle \nabla w(b) , a-b \rangle , 
\end{align*}
where  $w$ is continuously differentiable and $\alpha$-strongly convex function satisfying $\langle a - b , \nabla w(a) - \nabla w(b) \rangle \geq \alpha \| a - b \|^2$ for $a , b  \in \Theta $. 
When $w(\theta) = \frac{1}{2} || \theta ||_2^2$, i.e., $V(\theta, \theta^{(t)}) = \frac{1}{2} || \theta - \theta^{(t)}  ||_2^2$, (\ref{update form rspg}) is almost equal to (\ref{update form sgd}). 

Here, we denote two remarks on RSPG as a difference from the SGD. 
One is that the RSPG uses the mini-batch strategy, i.e., taking multiple samples at $t$-th time. 
The other is that the RSPG randomly select a final solution $\hat{\theta}$ from $\left\{ \theta^{(1)}, \ldots , \theta^{(T)} \right\}$ according to a certain probability distribution instead of taking the average of the iterates. 
This is because for non-convex stochastic optimization, later iterates does not always gather around local minimum and the average of the iterates can not work in such a convex case. 
%
%
\begin{algorithm}[!h]
\caption{Randomized stochastic projected gradient}\label{general RSPG}
\begin{algorithmic}
\Require{ The initial point $\theta^{(1)}$, the step size $\eta_t$, the mini-batch size $m_t$, the iteration limit $T$ and the probability mass function $P_{R}$ supported on $\left\{1,\ldots,T \right\}$.}
\State{Let $R$ be a random variable generated by probability mass function $P_{R}$.}
\For{ $t = 1 , \ldots , R-1$ }
\State{ $ \theta^{(t+1)} = \argmin_{\theta  \in \Theta} \left\langle \frac{1}{ m_t} \sum_{i=1}^{m_t} \nabla l((x_{t,i},y_{t,i});\theta^{(t)})   , \theta \right\rangle + \lambda P(\theta) +\frac{1}{ \eta_t} V( \theta , \theta^{(t)} )   .$ } 
\EndFor 
\Ensure{ $ \theta^{(R)}. $  }
\end{algorithmic}
\end{algorithm}

Next, we show the implementation of the RSPG, given by Algorithm\ref{general RSPG}. 
However, Algorithm \ref{general RSPG} has a large deviation of output because the  only one final output is selected via some probability mass function $P_R$. 
Therefore, \cite{Ghadimi:2016} also proposed the two phase RSPG (2-RSPG) which has the post-optimization phase. 
In the post-optimization phase, multiple outputs are selected and these are validated to determine the final output, as shown in Algorithm \ref{general 2-RSPG}.
%
%
\begin{algorithm}[!h]
\caption{Two phase randomized stochastic projected gradient}\label{general 2-RSPG}
\begin{algorithmic}
\Require{ The initial point $\theta^{(1)}$, the step size $\eta_t$, the mini-batch size $m_t$, the iteration limit $T$, the probability mass function $P_{R}$ supported on $\left\{1,\ldots,T \right\}$, the number of candidates $N_{cand}$ and the sample size $N_{post}$ for validation.}
\State{Let $R_{1}, R_{2},\ldots,R_{N_{cand}}$ be random variables generated by probability mass function $P_{R}$.}
\For{ $t = 1 , \ldots , \max{ \left\{R_{1}, R_{2},\ldots,R_{N_{cand}} \right\}}-1$ }
\State{ $ \theta^{(t+1)} = \argmin_{\theta  \in \Theta} \left\langle \frac{1}{ m_t} \sum_{i=1}^{m_t} \nabla l((x_{t,i},y_{t,i});\theta^{(t)})   , \theta \right\rangle + \lambda P(\theta) +\frac{1}{ \eta_t} V( \theta , \theta^{(t)} )   .$ } 
\EndFor 
\State{ \bf{Post-optimization phase:} } 
\State{ $ \theta^{(R_s)} = \argmin_{s=1,\ldots,N_{cand}}  \frac{1}{\eta_{R_s} } \|   \theta^{ \left( R_s \right) } - \theta^{ \left( R_{s}^{+} \right) }  \| ,$  } 
\State{where $ \theta^{ \left( R_{s}^{+} \right) } = \argmin_{\theta  \in \Theta}  \left\langle \frac{1}{N_{post}} \sum_{i=1}^{N_{post}} \nabla  l((x_{i},y_{i});\theta^{(R_s)})  , \theta \right\rangle +\lambda P(\theta) +\frac{1}{\eta_{R_s} } V(\theta, \theta^{(R_s)}) . $}
\Ensure{ $ \theta^{(R_s)}. $  }
\end{algorithmic}
\end{algorithm}
This can be expected to achieve a better complexity result of finding an $(\epsilon,\Lambda)-solution$, i.e., Prob$\left\{   C(\theta^{(R)})   \leq \epsilon \right\} \geq 1- \Lambda $, where $C$ is some convergence criterion, for some $\epsilon > 0 $ and $\Lambda \in (0,1)$. 
%
%
For more detailed descriptions and proofs, we refer to the Sect.4 in \cite{Ghadimi:2016}.

\section{Online robust and sparse GLM}\label{4}
In this section, we show the sparse $\gamma$-GLM with the stochastic optimization approach on three specific examples; linear regression,  logistic regression and Poisson regression with $L_1$ regularization. 
In what follows, we refer to the sparse $\gamma$-GLM with the stochastic optimization approach as the online sparse $\gamma$-GLM. 

In order to apply the RSPG to our methods (\ref{def sparse exp est}), we prepare the monotone transformation of the $\gamma$-cross entropy for regression in (\ref{def sparse exp est}) as follows
\begin{align}\label{rspg form gam}
\argmin_{\theta \in \Theta}  E_{g(x,y)} \left[ - \frac{   f(y|x;\theta)^{\gamma}  }{  \left( \int f(y|x;\theta)^{1+\gamma}dy\right)^\frac{\gamma}{1+\gamma}}  \right] + \lambda P(\theta) ,
\end{align}
and we suppose that $\Theta$ is $\mathbb{R}^n$ or closed ball with sufficiently large radius. 
Then, we can apply the RSPG to (\ref{rspg form gam}) and by virtue of (\ref{update form rspg}), the update formula takes the form 
\begin{align}\label{update gam rspg}
\theta^{(t+1)} = \argmin_{\theta \in \Theta}   \left\langle - \frac{1}{m_t} \sum_{i=1}^{m_t}  \nabla\frac{f(y_{t,i}|x_{t,i};\theta^{(t)})^{\gamma}}{\left( \int f(y|x_{t,i};\theta^{(t)})^{1+\gamma}dy\right)^{\frac{\gamma}{1+\gamma}}}  , \theta  \right\rangle + \lambda P(\theta)   +\frac{1}{ \eta_t} V( \theta , \theta^{(t)} )  .
\end{align}
More specifically, we suppose that $V(\theta, \theta^{(t)}) = \frac{1}{2} || \theta - \theta^{(t)} ||_2^2$  because the update formula can be obtained in closed form for some important sparse regularization terms, e.g., $L_1$ regularization, elasticnet. 
%
We illustrate the update algorithms based on Algorithm \ref{general RSPG} for three specific examples. 
The update algorithms based on Algorithm \ref{general 2-RSPG} are obtained in a similar manner. 

In order to implement our methods, we need to determine some tuning parameters, e.g., the step size $\eta_t$, mini-batch size $m_t$. 
In Sect.~\ref{Sect: convergence prop}, we discuss how to determine some tuning parameters in detail.

\subsection{Online sparse $\gamma$-linear regression}
Let $f(y|x;\theta)$ be the conditional density with $\theta=(\beta_0, \beta^T, \sigma^2)^T$, given by 
\begin{align*}
f(y|x;\theta)=\phi(y;\beta_0+x^T\beta,\sigma^2),
\end{align*} 
where $\phi(y;\mu,\sigma^2)$ is the normal density with mean parameter $\mu$ and variance parameter $\sigma^2$. 
Suppose that $P(\theta)$ is the $L_1$ regularization $||\beta ||_1$. 
Then, by virtue of (\ref{update gam rspg}), we can obtain the update formula given by 
\begin{align}\label{gam linear}
& \left( \beta_0^{(t+1)},\beta^{(t+1)},{\sigma^{2}}^{(t+1)} \right) \nonumber \\
& = \argmin_{\beta_0,\beta,\sigma^{2}}   \xi_1 (\beta_0^{(t)})  \beta_0  + \langle \xi_2 (\beta^{(t)}) , \beta \rangle + \xi_3 ( {\sigma^2}^{(t)})  \sigma^{2}  \nonumber \\
&  \qquad + \lambda \| \beta \|_1 + \frac{1}{2\eta_t}  \| \beta_0-\beta_0^{(t)} \|^{2}_2 +\frac{1}{2\eta_t}  \| \beta-\beta^{(t)} \|^{2}_2  +\frac{1}{2\eta_t}  \| \sigma^2-{\sigma^2}^{(t)} \|^{2}_2 ,
\end{align}
where
\begin{align*}
 \xi_1(\beta_0^{(t)}) &=  - \frac{1}{m_t} \sum_{i=1}^{m_t} \left[  \frac{\gamma (y_{t,i} -\beta_0^{(t)} - {x_{t,i}}^T \beta^{(t)} )}{ {\sigma^2}^{(t)} } \left( \frac{1+\gamma}{2 \pi {\sigma^2}^{(t)}} \right)^{\frac{\gamma}{2(1+\gamma)}} \exp \left\{- \frac{\gamma (y_{t,i} -\beta_0^{(t)} - {x_{t,i}}^T \beta^{(t)})^2 }{2 {\sigma^2}^{(t)}} \right\} \right], \\
\xi_2(\beta^{(t)}) &=  - \frac{1}{m_t} \sum_{i=1}^{m_t}  \left[ \frac{\gamma (y_{t,i} -\beta_0^{(t)} - {x_{t,i}}^T \beta^{(t)} )}{ {\sigma^2}^{(t)} } \left( \frac{1+\gamma}{2 \pi {\sigma^2}^{(t)}} \right)^{\frac{\gamma}{2(1+\gamma)}} \exp \left\{- \frac{\gamma (y_{t,i} -\beta_0^{(t)} - {x_{t,i}}^T \beta^{(t)})^2 }{2 {\sigma^2}^{(t)}} \right\} x_{t,i} \right] , \\
\xi_3( {\sigma^2}^{(t)})  &=  \frac{1}{m_t} \sum_{i=1}^{m_t} \left[  \frac{\gamma}{2} \left( \frac{1+\gamma}{2 \pi {\sigma^2}^{(t)} } \right)^{\frac{\gamma}{2(1+\gamma)}} \left\{ \frac{1}{(1+\gamma) {\sigma^2}^{(t)} }  - \frac{(y_{t,i} -\beta_0^{(t)} - {x_{t,i}}^T \beta^{(t)} )^2}{ {\sigma^4}^{(t)}}  \right\} \right. \\
& \left.  \qquad \qquad \exp \left\{ - \frac{\gamma(y_{t,i} - \beta_0^{(t)} - {x_{t,i}}^T \beta^{(t)} )^2 }{2 {\sigma^2}^{(t)}} \right\}  \right]  .
\end{align*}
Consequently, we can obtain the update algorithm, as shown in Algorithm \ref{algo online gam reg}. 
\begin{algorithm}[!h]
\caption{Online sparse $\gamma$-linear regression}\label{algo online gam reg}
\begin{algorithmic}
\Require{ The initial points $\beta_0^{(1)}, \ \beta^{(1)}, \ {\sigma^2}^{(1)}$, the step size $\eta_t$, the mini-batch size $m_t$, the iteration limit $T$ and the probability mass function $P_{R}$ supported on $\left\{1,\ldots,T \right\}$.}
\State{Let $R$ be a random variable generated by probability mass function $P_{R}$.}
\For{ $t = 1  , \ldots , R-1$ }
\State{ $ \beta_0^{(t+1)} = \beta_0^{(t)} - \eta_t \xi_1(\beta_0^{(t)}) .$ } 
\State{ $ \beta_j^{(t+1)} = S (\beta_j^{(t)} - \eta_{t} {\xi_{2}}_{j} (\beta^{(t)}), \eta_{t} \lambda)   \ \ (j=1,\ldots,p) .$ } 
\State{ $ {\sigma^2}^{(t+1)} = {\sigma^2}^{(t)} - \eta_t \xi_3({\sigma^2}^{(t)}) .$ }
\EndFor 
\Ensure{ $ \beta_0^{(R)}, \ \beta^{(R)}, \ {\sigma^{2}}^{(R)} .$  }
\end{algorithmic}
\end{algorithm}

Here, we briefly show the robustness of online sparse $\gamma$-linear regression. 
For simplicity, we consider the intercept parameter $\beta_0$. 
Suppose that the $(x_{t,k},y_{t,k})$ is an outlier at $t$-th time. 
The conditional probability density $f(y_{t,k}|x_{t,k};\theta^{(t)})$ can be expected to be sufficiently small. 
We see from $ f(y_{t,k}|x_{t,k};\theta^{(t)}) \approx 0$ and (\ref{gam linear}) that
\begin{align*}
& \beta_0^{(t+1)} \nonumber \\
&= \argmin_{\beta_0}  - \frac{1}{m_t} \sum_{1 \leq i \neq k \leq m_t}  \left[ \frac{\gamma (y_{t,i} -\beta_0^{(t)} - {x_{t,i}}^T \beta^{(t)} )}{ {\sigma^2}^{(t)} } \left( \frac{1+\gamma}{2 \pi {\sigma^2}^{(t)}} \right)^{\frac{\gamma}{2(1+\gamma)}} \exp \left\{- \frac{\gamma (y_{t,i} -\beta_0^{(t)} - {x_{t,i}}^T \beta^{(t)})^2 }{2 {\sigma^2}^{(t)}} \right\} \right] \times  \beta_0  \nonumber \\
& \quad    - \frac{1}{m_t}   \frac{\gamma (y_{t,k} -\beta_0^{(t)} - {x_{t,k}}^T \beta^{(t)} )}{ {\sigma^2}^{(t)} } \left( \frac{1+\gamma}{2 \pi {\sigma^2}^{(t)}} \right)^{\frac{\gamma}{2(1+\gamma)}} \exp \left\{- \frac{\gamma (y_{t,k} -\beta_0^{(t)} - {x_{t,k}}^T \beta^{(t)})^2 }{2 {\sigma^2}^{(t)}} \right\}   \times  \beta_0  \\
& \qquad  + \frac{1}{2\eta_t}  \| \beta_0-\beta_0^{(t)} \|^{2}_2   \\
&= \argmin_{\beta_0}  - \frac{1}{m_t} \sum_{1 \leq i \neq k \leq m_t}  \left[ \frac{\gamma (y_{t,i} -\beta_0^{(t)} - {x_{t,i}}^T \beta^{(t)} )}{ {\sigma^2}^{(t)} } \left( \frac{1+\gamma}{2 \pi {\sigma^2}^{(t)}} \right)^{\frac{\gamma}{2(1+\gamma)}} \exp \left\{- \frac{\gamma (y_{t,i} -\beta_0^{(t)} - {x_{t,i}}^T \beta^{(t)})^2 }{2 {\sigma^2}^{(t)}} \right\} \right] \times  \beta_0  \nonumber \\
& \quad   \underset{\approx 0}{ \underline{  - \frac{1}{m_t}   \frac{ \gamma (1+\gamma)^{\frac{\gamma}{2(1+\gamma)}}  (y_{t,k} -\beta_0^{(t)} - {x_{t,k}}^T \beta^{(t)} )}{ {\sigma^2}^{(t)} } \left( 2 \pi {\sigma^2}^{(t)}  \right)^{\frac{\gamma^2}{2(1+\gamma)}} f(y_{t,k} | x_{t,k} ; \theta^{(t)})^{\gamma}   } } \times  \beta_0  \\
& \qquad  + \frac{1}{2\eta_t}  \| \beta_0-\beta_0^{(t)} \|^{2}_2   \\
\end{align*}
Therefore, the effect of an outlier is naturally ignored in (\ref{gam linear}). 
Similarly, we can also see the robustness for parameters $\beta$ and $\sigma^2$. 

\subsection{Online sparse $\gamma$-logistic regression}
Let $f(y|x;\theta)$ be the conditional density with $\theta=(\beta_0,\beta^T)^T$, given by 
\begin{align*}
f(y|x;\beta_0,\beta)=F( \tilde{x}^T \theta )^y (1 - F(\tilde{x}^T \theta) )^{(1-y)} ,
\end{align*}
where $\tilde{x}=(1,x^T)^T$ and $F(u)= \frac{1}{1+\exp(-u)}$.
Then, by virtue of (\ref{update gam rspg}), we can obtain the update formula given by
\begin{align}\label{gam logistic}
& \left( \beta_0^{(t+1)}, \beta^{(t+1)} \right) \nonumber \\
& \argmin_{\beta_0, \beta}  \nu_1(\beta_0^{(t)})    \beta_0 + \langle  \nu_2(\beta^{(t)}) , \beta \rangle  + \lambda || \beta ||_1  +\frac{1}{2\eta_t}  \| \beta_0-\beta_0^{(t)} \|^{2}_2  +\frac{1}{2\eta_t}  \| \beta-\beta^{(t)} \|^{2}_2 ,
\end{align}
where
\begin{align*}
\nu_1(\beta_0^{(t)}) &=  - \frac{1}{m_t} \sum_{i=1}^{m_t} \left[  \frac{ \gamma \exp( \gamma y_{t,i} \tilde{x}_{t,i}^T  \theta^{(t)} ) \left\{  y_{t,i} - \frac{ \exp( (1+\gamma) \tilde{x}_{t,i}^T  \theta^{(t)} )   }{ 1+\exp( (1+\gamma) \tilde{x}_{t,i}^T  \theta^{(t)} )   }  \right\}  }{ \left\{ 1+\exp( (1+\gamma) \tilde{x}_{t,i}^T  \theta^{(t)} )  \right\}^{\frac{\gamma}{1+\gamma}} } \right] , \\
\nu_2(\beta^{(t)}) &= - \frac{1}{m_t} \sum_{i=1}^{m_t} \left[  \frac{ \gamma \exp( \gamma y_{t,i} \tilde{x}_{t,i}^T  \theta^{(t)} ) \left\{  y_{t,i} - \frac{ \exp( (1+\gamma) \tilde{x}_{t,i}^T  \theta^{(t)} )   }{ 1+\exp( (1+\gamma) \tilde{x}_{t,i}^T  \theta^{(t)} )   }  \right\}  }{ \left\{ 1+\exp( (1+\gamma) \tilde{x}_{t,i}^T  \theta^{(t)} )  \right\}^{\frac{\gamma}{1+\gamma}} }   x_{t,i} \right] .
\end{align*}
Consequently, we can obtain the update algorithm as shown in Algorithm \ref{algo online gam logistic reg}. 
\begin{algorithm}[!h]
\caption{Online sparse $\gamma$-logistic regression}\label{algo online gam logistic reg}
\begin{algorithmic}
\Require{ The initial points $\beta_0^{(1)}, \ \beta^{(1)}$, the step size $\eta_t$, the mini-batch size $m_t$, the iteration limit $T$ and the probability mass function $P_{R}$ supported on $\left\{1,\ldots,T \right\}$.}
\State{Let $R$ be a random variable generated by probability mass function $P_{R}$.}
\For{ $t = 1  , \ldots , R-1$ }
\State{ $ \beta_0^{(t+1)} = \beta_0^{(t)} - \eta_t \nu_1(\beta_0^{(t)}) .$ } 
\State{ $ \beta_j^{(t+1)} = S (\beta_j^{(t)} - \eta_{t} {\nu_{2}}_{j} (\beta^{(t)}), \eta_{t} \lambda)   \ \ (j=1,\ldots,p). $ } 
\EndFor 
\Ensure{ $ \beta_0^{(R)}, \ \beta^{(R)} .$  }
\end{algorithmic}
\end{algorithm}
In a similar way to online sparse $\gamma$-linear regression, we can also see the robustness for parameters $\beta_0$ and $\beta$ in online sparse $\gamma$-logistic regression (\ref{gam logistic}). 

\subsection{Online sparse $\gamma$-Poisson regression}\label{update algo:online gam poisson}
Let $f(y|x;\theta)$ be the conditional density with $\theta=(\beta_0,\beta^T)^T$, given by 
\begin{align*}
f(y|x;\theta) =  \frac{\exp(-\mu_{x}(\theta)  ) }{y!} \mu_{x}(\theta)^y,
\end{align*}
where $\mu_{x}(\theta) = \mu_{x}(\beta_0,\beta) = \exp(\beta_0+x^T\beta)$. 
Then, by virtue of (\ref{update gam rspg}), we can obtain the update formula given by
\begin{align}\label{gam poisson}
& \left( \beta_0^{(t+1)}, \beta^{(t+1)} \right) \nonumber \\
& =\argmin_{\beta_0, \beta}  \zeta_1(\beta_0^{(t)})  \beta_0  + \langle \zeta_{2}(\beta^{(t)})   , \beta \rangle + \lambda || \beta ||_1  +\frac{1}{2\eta_t}  \| \beta_0-\beta_0^{(t)} \|^{2}_2  +\frac{1}{2\eta_t}  \| \beta-\beta^{(t)} \|^{2}_2 , 
\end{align}
where
\begin{align*}
\zeta_1(\beta_0^{(t)}) &= \frac{1}{m_t} \sum_{i=1}^{m_t} \left[  \frac{ \gamma f(y_{t,i}|x_{t,i};\theta^{(t)})^\gamma  \left\{ \sum_{y=0}^\infty (y - y_{t,i})  f(y|x_{t,i};\theta^{(t)})^{1+\gamma} \right\}   }{  \left\{ \sum_{y=0}^\infty f(y|x_{t,i};\theta^{(t)})^{1+\gamma} \right\}^{\frac{1+2\gamma}{1+\gamma}} }  \right], \\
\zeta_{2}(\beta^{(t)}) &= \frac{1}{m_t} \sum_{i=1}^{m_t} \left[  \frac{ \gamma f(y_{t,i}|x_{t,i};\theta^{(t)})^\gamma  \left\{ \sum_{y=0}^\infty (y - y_{t,i})  f(y|x_{t,i};\theta^{(t)})^{1+\gamma} \right\}   }{  \left\{ \sum_{y=0}^\infty f(y|x_{t,i};\theta^{(t)})^{1+\gamma} \right\}^{\frac{1+2\gamma}{1+\gamma}} }  x_{t,i}  \right] .
\end{align*}

In (\ref{gam poisson}), two types hypergeometric series exist. 
Here, we prove a convergence of $\sum_{y=0}^\infty f(y|x_{t,i};\theta^{(t)})^{1+\gamma}$ and $\sum_{y=0}^\infty  (y- y_{t,i} ) f(y|x_{t,i};\theta^{(t)})^{1+\gamma}$.
First, let us consider $\sum_{y=0}^\infty f(y|x_{t,i};\theta^{(t)})^{1+\gamma}$ and we denote $n$-th term that $ S_n = f(n|x_{t,i};\theta^{(t)})^{1+\gamma} $. 
Then, we use the dalembert ratio test for $S_n$:
\begin{align*}
&\lim_{n \to \infty} \left| \frac{ S_{n+1}} {S_n} \right| \\
&=\lim_{n \to \infty} \left| \frac{ f(n+1|x_{t,i};\theta^{(t)})^{1+\gamma} } {f(n|x_{t,i};\theta^{(t)})^{1+\gamma}} \right| \\
&=\lim_{n \to \infty} \left| \frac{  \frac{\exp(-\mu_{x_{t,i}}(\beta_0^{(t)}, \beta^{(t)}) ) }{n+1!} \mu_{x_{t,i}}(\beta_0^{(t)}, \beta^{(t)})^{n+1} } { \frac{\exp(-\mu_{x_{t,i}}(\beta_0^{(t)}, \beta^{(t)}) ) }{n!} \mu_{x_{t,i}}(\beta_0^{(t)}, \beta^{(t)})^n  } \right|^{1+\gamma} \\
&=\lim_{n \to \infty} \left| \frac{    \mu_{x_{t,i}}(\beta_0^{(t)}, \beta^{(t)}) } { n+1   } \right|^{1+\gamma} \\
&\mbox{If the term }  \mu_{x_{t,i}}(\beta_0^{(t)}, \beta^{(t)}) \mbox{ is bounded,} \\
& = 0.
\end{align*}
Therefore, $\sum_{y=0}^\infty f(y|x_{t,i};\theta^{(t)})^{1+\gamma}$ converges. 

Next, let us consider $\sum_{y=0}^\infty  (y - y_{t,i} ) f(y|x_{t,i};\theta^{(t)})^{1+\gamma}$ and we denote $n$-th term that $ S^{'}_n =(n - y_{t,i} ) f(n|x_{t,i};\theta^{(t)})^{1+\gamma}$. 
Then, we use the dalembert ratio test for $S^{'}_n$:
\begin{align*}
&\lim_{n \to \infty} \left| \frac{ S^{'}_{n+1}} {S^{'}_n} \right| \\
&=  \lim_{n \to \infty} \left| \frac{ (1+\frac{1}{n} - \frac{y_{t,i}}{n} )f(n+1|x_{t,i};\theta^{(t)})^{1+\gamma} } {  (1 - \frac{y_{t,i}}{n} )f(n|x_{t,i};\theta^{(t)})^{1+\gamma}  } \right| \\
& =  \lim_{n \to \infty} \left| \frac{ (1+\frac{1}{n} - \frac{y_{t,i}}{n} )} {  (1 - \frac{y_{t,i}}{n} ) } \right| \left| \frac{ f(n+1|x_{t,i};\theta^{(t)})^{1+\gamma} } {  f(n|x_{t,i};\theta^{(t)})^{1+\gamma}  } \right| \\
& = 0.
\end{align*}
Therefore, $\sum_{y=0}^\infty  (y-y_{t,i}) f(y|x_{t,i};\theta^{(t)})^{1+\gamma}$ converges. 

Consequently, we can obtain the update algorithm as shown in Algorithm \ref{algo online gam poisson reg}. 
\begin{algorithm}[!h]
\caption{Online sparse $\gamma$-Poisson regression}\label{algo online gam poisson reg}
\begin{algorithmic}
\Require{ The initial points $\beta_0^{(1)}, \ \beta^{(1)}$, the step size $\eta_t$, the mini-batch size $m_t$, the iteration limit $T$ and the probability mass function $P_{R}$ supported on $\left\{1,\ldots,T \right\}$.}
\State{Let $R$ be a random variable generated by probability mass function $P_{R}$.}
\For{ $t = 1  , \ldots , R-1$ }
\State{ $ \beta_0^{(t+1)} = \beta_0^{(t)} - \eta_t \zeta_1(\beta_0^{(t)}) .$ } 
\State{ $ \beta_j^{(t+1)} = S (\beta_j^{(t)} - \eta_{t} {\zeta_{2}}_{j} (\beta^{(t)}), \eta_{t} \lambda)   \ \ (j=1,\ldots,p) .$ } 
\EndFor 
\Ensure{ $ \beta_0^{(R)}, \ \beta^{(R)} .$  }
\end{algorithmic}
\end{algorithm}
In a similar way to online sparse $\gamma$-linear regression, we can also see the robustness for parameters $\beta_0$ and $\beta$ in online sparse $\gamma$-Poisson regression (\ref{gam poisson}). 
Moreover, this update algorithm requires at most twice sample size $2n=2 \times \sum_{t=1}^T m_t$ times of an approximate calculation for the hypergeometric series in Algorithm \ref{algo online gam poisson reg}. 
Therefore, we can achieve a significant reduction in computational complexity. 

\section{Convergence property of online sparse $\gamma$-GLM}\label{Sect: convergence prop}
In this section, we show the global convergence property of the RSPG established by \cite{Ghadimi:2016}. 
Moreover, we extend it to the classical first-order necessary condition, i.e., at a local minimum, the directional derivative, if it exists, is non-negative for any direction (see, e.g., \cite{borwein2010convex}). 

First, we show the global convergence property of the RSPG. 
In order to apply to online sparse $\gamma$-GLM, we slightly modify some notations. 
%
%
We consider the following optimization problem (\ref{RERM def}) again:
\begin{align*}
\Psi^* \coloneqq \min_{\theta \in \Theta}  \underset{ \coloneqq \Psi (\theta) }{ \underline{ E_{(x,y)} \left[ l((x,y);\theta) \right] + \lambda P(\theta) } },
\end{align*}
where $E_{(x,y)} \left[ l((x,y);\theta) \right]$ is continuously differentiable and possibly non-convex. 
The update formula (\ref{update form rspg}) of the RSPG is as follows:
\begin{align*}
 \theta^{(t+1)} = \argmin_{ \theta  \in \Theta} \left\langle \frac{1}{ m_t} \sum_{i=1}^{m_t} \nabla l((x_{t,i},y_{t,i});\theta^{(t)} )  , \theta \right\rangle + \lambda P(\theta)  +\frac{1}{ \eta_t} V( \theta , \theta^{(t)} )  ,
\end{align*}
where
\begin{align*}
V(a,b) = w(a) -w(b) - \langle \nabla w(b) , a-b \rangle , 
\end{align*}
and $w$ is continuously differentiable and $\alpha$-strongly convex function satisfying $\langle a - b , \nabla w(a) - \nabla w(b) \rangle \geq \alpha \| a - b \|^2$ for $a , b  \in \Theta $. 
%
We make the following assumptions. 

{\bf Assumption 1} 
$\nabla  E_{(x,y)} \left[ l((x,y);\theta) \right]$ is $L$-Lipschitz continuous for some $L>0$, i.e., 
\begin{align}\label{assum1}
\| \nabla  E_{(x,y)} \left[ l((x,y);\theta_1) \right] - \nabla  E_{(x,y)} \left[ l((x,y);\theta_2) \right] \| < L \|  \theta_1 - \theta_2  \|,  \mbox{ for  any}  \ \theta_1, \theta_2 \ \in \Theta.
\end{align}

{\bf Assumption 2}
For any $t \geq 1$,
\begin{align}
&E_{(x_t,y_t)}  \left[  \nabla l((x_t,y_t);\theta^{(t)})  \right]  = \nabla  E_{(x_t,y_t)} \left[ l((x_t,y_t);\theta^{(t)}) \right], \\ 
&E_{(x_t,y_t)}  \left[  \left\|  \nabla l((x_t,y_t);\theta^{(t)})  -\nabla  E_{(x_t,y_t)} \left[ l((x_t,y_t);\theta^{(t)}) \right] \right\|^2 \right] \leq \tau^2, \label{assum2}
\end{align}
where $\tau > 0 $ is a constant. 

Let us define 
\begin{align*}
P_{X,R} &= \frac{1}{\eta_{R}} \left( \theta^{(R)} -\theta^{+} \right) , \\
\tilde{P}_{X,R} &= \frac{1}{\eta_{R}} \left( \theta^{(R)} -\tilde{\theta}^{+} \right) ,
\end{align*}
where 
\begin{align}\label{update true grad RSPG}
\theta^{+} &= \argmin_{\theta  \in \Theta} \left\langle \nabla E_{(x,y)} \left[ l((x,y);  \theta^{(R)} ) \right]   , \theta \right\rangle + \lambda P(\theta)  +\frac{1}{\eta_R} V( \theta , \theta^{(R)} ),  \\
\tilde{\theta}^{+} &= \argmin_{\theta \in \Theta} \left\langle  \frac{1}{ m_{R} } \sum_{i=1}^{m_{R}} \nabla  l((x_{R,i},y_{R,i});  \theta^{(R)} )  , \theta \right\rangle + \lambda P(\theta)  +\frac{1}{\eta_R} V( \theta , \theta^{(R)} ). \nonumber 
\end{align}
Then, the following global convergence property was obtained. 
\begin{theorem}\label{RSPG glb conv}
{\bf [Global Convergence Property in \cite{Ghadimi:2016}] } 

\noindent Suppose that the step sizes $\left\{ \eta_{t} \right\}$ are chosen such that $0< \eta_{t} \leq \frac{\alpha}{L}$ with $\eta_{t} < \frac{\alpha}{L}$ for at least one $t$, and the probability mass function $P_{R}$ is chosen such that for any $t=1,\ldots,T$,
\begin{align}\label{choose prob}
P_{R}(t) \coloneqq \mbox{Prob} \left\{ R=t \right\} = \frac{\alpha \eta_{t} - L \eta_{t}^2 }{ \sum_{t=1}^T  \left( \alpha \eta_{t} - L \eta_{t}^2  \right)}.
\end{align}
Then, we have
\begin{align*}
E \left[ || \tilde{P}_{X,R}  ||^2  \right]  \leq  \frac{ L D^2_{\Psi} + \left( \tau^2 / \alpha \right)   \sum_{t=1}^T \left( \eta_{t} / m_{t} \right) }{ \sum_{t=1}^T  \left( \alpha \eta_{t}   - L \eta_{t}^2  \right) }  ,
\end{align*}
where the expectation was taken with respect to $R$ and past samples $(x_{t,i}, y_{t,i}) \ (t=1,\ldots,T; \ i=1,\ldots, m_{t} )$ and $D_{ \Psi }= \left[ \frac{ \Psi( \theta^{(1)} ) - \Psi^*  }{L} \right]^{\frac{1}{2}} $.
\end{theorem}
\begin{proof}
See \cite{Ghadimi:2016}, Theorem 2. 
\end{proof}

In particular, \cite{Ghadimi:2016} investigated the constant step size and mini-batch size policy as follows.
\begin{corollary}\label{bound for diff est}
{\bf [Global Convergence Property with constant step size and mini-batch size in \cite{Ghadimi:2016}] } 

\noindent Suppose that the step sizes and mini-batch sizes are $\eta_{t} = \frac{\alpha}{2L}$ and $m_t =m \ (\geq1)$ for all $t=1, \ldots, T$, and the probability mass function $P_{R}$ is chosen as (\ref{choose prob}). Then, we have 
\begin{align*}
E \left[ \| \tilde{P}_{X,R} \|^2 \right] \leq \frac{4 L^2 D^2_{\Psi}}{\alpha^2 T} +   \frac{2 \tau^2}{ \alpha^2 m} \ \ and \ \ E \left[ \| P_{X,R} \|^2 \right] \leq  \frac{8 L^2 D^2_{\Psi}}{\alpha^2 T} +   \frac{6 \tau^2}{ \alpha^2 m} .
\end{align*}
Moreover, the appropriate choice of mini-batch size $m$ is given by 
\begin{align*}
m = \left\lceil \min \left\{ \max \left\{ 1 , \frac{\tau \sqrt{6 N } }{4 L \tilde{D} }  \right\} , N  \right\} \right\rceil  ,
\end{align*}
where $\tilde{D} >0$ and $N\left( = m \times T\right)$ is the number of total samples. 
Then, with the above setting, we have the following result 
\begin{align}\label{RSPG global conv}
\frac{\alpha^2}{L} E \left[ \| P_{X,R} \|^2 \right] \leq  \frac{16 L D^2_{\Psi} }{N} + \frac{4 \sqrt{6} \tau }{ \sqrt{ N}} \left( \frac{D^2_{\Psi} }{\tilde{D}} + \tilde{D} \max \left\{ 1, \frac{\sqrt{6} \tau}{ 4 L \tilde{D} \sqrt{N} } \right\} \right) .
\end{align}
Furthermore, when $N$ is relatively large, the optimal choice of $\tilde{D}$ would be $D_{\Psi}$ and (\ref{RSPG global conv}) reduces to 
\begin{align*}
\frac{\alpha^2}{L} E \left[ \| P_{X,R} \|^2 \right] \leq  \frac{16 L D^2_{\Psi} }{N} + \frac{8 \sqrt{6}  D_{\Psi} \tau }{ \sqrt{ N}}  .
\end{align*}
\end{corollary}
\begin{proof}
See \cite{Ghadimi:2016}, Corollary 4. 
\end{proof}

Finally, we extend (\ref{RSPG global conv}) to the classical first-order necessary  condition as follows
\begin{theorem}
{\bf The Modified Global Convergence Property} 

\noindent Under the same assumptions in Theorem \ref{RSPG glb conv}, we can expect $P_{X,R} \approx 0$ with high probability from (\ref{RSPG global conv}) and Markov inequality.  
Then, for any direction $\delta$ and $\theta^{(R)} \in \ re \ int \left( \Theta \right)$, we have
\begin{align}\label{mod RSPG glb conv}
\Psi^{'}( \theta^{(R)}; \delta) = \lim_{ k \downarrow 0} \frac{ \Psi(\theta^{(R)} + k\delta) - \Psi(\theta^{(R)}) }{k} \geq 0 \ with \ high \ probbility.
\end{align}
%
%
%
\end{theorem} 
The proof is in Appendix. 
Under the above assumptions and results, online sparse $\gamma$-GLM has the global convergence property. 
Therefore, we adopted the following parameter setting in online sparse $\gamma$-GLM:
\begin{align*}
\mbox{step size: }\eta_{t} &=  \frac{1}{2L}, \\
\mbox{mini-batch size: }m_{t} &= \left\lceil \min \left\{ \max \left\{ 1 , \frac{\tau \sqrt{6 N } }{4 L \tilde{D} }  \right\} , N  \right\} \right\rceil .
\end{align*}
More specifically, when the (approximate) minimum value of the objective function $\Psi^*$ is known, e.g., the objective function is non-negative, we should use $D_{\Psi}$ instead of $\tilde{D}$. 
In numerical experiment, we used the $D_{\Psi}$ because we can obtain $\Psi^*$ in advance. 
In real data analysis, we can not obtain $\Psi^*$ in advance. 
Then, we used the some values of $\tilde{D}$, i.e., the some values of mini-batch size $m_{t}$.

%

\section{Numerical experiment}\label{6}
In this section, we present the numerical results of online sparse $\gamma$-linear regression. 
We compared online sparse $\gamma$-linear regression based on the RSPG with online sparse $\gamma$-linear regression based on the SGD, which does not guarantee convergence for non-convex case. 
The RSPG has two variants, which are shown in Algorithms \ref{general RSPG} and \ref{general 2-RSPG}. 
In this experiment, we adopted the 2-RSPG for the numerical stability. 
In what follows, we refer to the 2-RSPG as the RSPG. 
As a comparative method, we implemented the SGD with the same parameter setting described in Sect \ref{SGD sect}. 
All results were obtained in R version 3.3.0 with Intel Core i7-4790K machine. 
%
%
%

\subsection{Linear regression models for simulation}\label{Sect: Reg mod for sim}
We used the simulation model given by 

\begin{gather*}
y=\beta_0+\beta_1 x_1 + \beta_2 x_2+ \cdots +\beta_p x_p + e, \quad e \sim N(0,0.5^2). 
\end{gather*}
The sample size and the number of explanatory variables were set to be $N=10000, 30000$ and $p=1000, 2000$, respectively. 
The true coefficients were given by 
\begin{gather*}
\beta_1 =1 ,\ \beta_2 = 2,\ \beta_4 = 4,\ \beta_7 = 7,\ \beta_{11} =11,\nonumber \\
 \beta_j =0 \ \mbox{for} \ j \in \{0, \ldots ,p \} \backslash \{1,2,4,7,11\}. 
\end{gather*}
We arranged a broad range of regression coefficients to observe sparsity for various degrees of regression coefficients. 
The explanatory variables were generated from a normal distribution $N(0,\Sigma)$ with $\Sigma=(0.2^{|i-j|})_{1 \leq i,j \leq p }$. 
We generated 30 random samples. 

Outliers were incorporated into simulations. 
We set the outlier ratio ($\epsilon=0.2$) and the outlier pattern that the outliers were generated around the middle part of the explanatory variable, where the explanatory variables were generated from $N(0,0.5^2)$ and the error terms were generated from $N(20,0.5^2)$. 
%
%

\subsection{Performance measure}\label{per meas}
The empirical regularized risk and the (approximated) expected regularized risk were used to verify the fitness of regression:
\begin{align*}
\textrm{EmpRisk} &= \frac{1}{N}  \sum_{i=1}^N - \frac{   f(y_i|x_i;\hat{\theta})^{\gamma}  }{  \left( \int f(y|x_i;\hat{\theta})^{1+\gamma}dy\right)^\frac{\gamma}{1+\gamma}}   + \lambda \| \hat{\beta} \|_1 , \\
\textrm{ExpRisk}  &=  E_{g(x,y)} \left[ - \frac{   f(y|x;\hat{\theta})^{\gamma}  }{  \left( \int f(y|x;\hat{\theta})^{1+\gamma}dy\right)^\frac{\gamma}{1+\gamma}}  \right] + \lambda \| \hat{\beta} \|_1 \\
& \approx \frac{1}{N_{test}} \sum_{i=1}^{N_{test}}  - \frac{   f(y^*_i|x^*_i;\hat{\theta})^{\gamma}  }{  \left( \int f(y|x^*_i;\hat{\theta})^{1+\gamma}dy\right)^\frac{\gamma}{1+\gamma}}   + \lambda \| \hat{\beta} \|_1, 
\end{align*}
where $f(y|x;\hat{\theta})=\phi(y;\hat{\beta_0}+x^T\hat{\beta},\hat{\sigma}^2)$ and $(x_i^*, y_i^*)$ ($i=1, \ldots ,N_{test}$) is test samples generated from the simulation model with outlier scheme. 
In this experiment, we used $N_{test}=70000$. 
%

%
%
%
%

%

\subsection{Initial point and tuning parameter}\label{det init}
In our method, we need an initial point and some tuning parameters to obtain the estimate. 
Therefore, we used $N_{init}=200$ samples which were used for estimating an initial point and other parameters $L$ in (\ref{assum1}) and $\tau^2$ in (\ref{assum2}) to calculate in advance. 
We suggest the following ways to prepare an initial point. 
The estimate of other conventional robust and sparse regression methods would give a good initial point. 
For another choice, the estimate of the RANSAC (random sample consensus) algorithm would also give a good initial point. 
In this experiment, we added the noise to the estimate of the RANSAC and used it as an initial point. 

For estimating $L$ and $\tau^2$, we followed the way to in Sect. 6 of \cite{Ghadimi:2016}. 
Moreover, we used the following value of tuning parameters in this experiment. 
The parameter $\gamma$ in the $\gamma$-divergence was set to $0.1$. 
The parameter $\lambda$ of $L_1$ regularization was set to $10^{-1}, 10^{-2}, 10^{-3}$.

The RSPG needed the number of candidates $N_{cand}$ and post-samples $N_{post}$ for post-optimization as described in Algorithm \ref{general 2-RSPG}. 
Then, we used $N_{cand}=5$ and $N_{post}= \left\lceil N/10 \right\rceil$. 
%
%
%
\subsection{Result}
Tables \ref{tab1}-\ref{tab3} show the EmpRisk, ExpRisk and computation time in the case $\lambda = 10^{-3}, 10^{-2} \mbox{ and } 10^{-1}$. 
%
%
%
Except for the computation time, our method outperformed comparative methods with several sizes of sample and dimension. 
We verify that the SGD, which are not theoretically guaranteed to converge for non-convex loss, can not reach the stationary point numerically. 
For the computation time, our method was comparable to the SGD. 
%

\begin{table}[H]
\caption{EmpRisk, ExpRisk and computation time for $\lambda=10^{-3}$ }
\label{tab1}
 \centering
 \begin{tabular}{c | c c c  }    
 \multicolumn{1}{c}{} & \multicolumn{3}{c}{$N=10000$, $p=1000$ } \\
Methods  & EmpRisk & ExpRisk  & Time   \\ \hline
RSPG  & -0.629   & -0.628      & 75.2 \\
SGD with $1$ mini-batch & -0.162 & -0.155 & 95.9 \\
SGD with $10$ mini-batch & 1.1$\times10^{-2}$ & 1.45$\times10^{-2}$  & 73.2 \\
SGD with $30$ mini-batch & 4.79$\times10^{-2}$ & 5.02$\times10^{-2}$  & 71.4 \\
SGD with $50$ mini-batch & 6.03$\times10^{-2}$ & 6.21$\times10^{-2}$  & 71.1 \\

  \multicolumn{1}{c}{} & \multicolumn{3}{c}{$N=30000$, $ p=1000$}  \\
Methods  & EmpRisk & ExpRisk   & Time \\ \hline
RSPG  & -0.692 & -0.691   & 78.3 \\
SGD with $1$ mini-batch & -0.365 & -0.362  & 148 \\
SGD with $10$ mini-batch & -0.111 & -0.11  & 79.6 \\
SGD with $30$ mini-batch & -5.71$\times10^{-2}$ & -5.6$\times10^{-2}$ & 73.7 \\
SGD with $50$ mini-batch & -3.98$\times10^{-2}$ & -3.88$\times10^{-2}$  & 238 \\

  \multicolumn{1}{c}{} & \multicolumn{3}{c}{$N=10000$, $ p=2000$}  \\
Methods  & EmpRisk & ExpRisk  & Time  \\ \hline
RSPG  & -0.646 & -0.646   & 117 \\
SGD with $1$ mini-batch & 0.187 & 0.194 & 145 \\
SGD with $10$ mini-batch & 0.428 & 0.431  & 99.2 \\
SGD with $30$ mini-batch & 0.479 & 0.481 & 95.7 \\
SGD with $50$ mini-batch & 0.496 & 0.499  & 166 \\

  \multicolumn{1}{c}{} & \multicolumn{3}{c}{$N=30000$, $ p=2000$}  \\
Methods  & EmpRisk & ExpRisk  & Time  \\ \hline
RSPG  & -0.696 & -0.696   & 125 \\
SGD with $1$ mini-batch & -3.89$\times10^{-2}$ & -3.56$\times10^{-2}$  & 251 \\
SGD with $10$ mini-batch & 0.357 & 0.359  & 112 \\
SGD with $30$ mini-batch & 0.442 & 0.443  & 101 \\
SGD with $50$ mini-batch & 0.469 & 0.47  & 337 \\

\end{tabular}
\end{table}

\begin{table}[H]
\caption{EmpRisk, ExpRisk and computation time for $\lambda=10^{-2}$  }
\label{tab2} 
 \centering
 \begin{tabular}{c | c c c  }    
 \multicolumn{1}{c}{} & \multicolumn{3}{c}{$N=10000$, $p=1000$ } \\
Methods  & EmpRisk & ExpRisk  & Time   \\ \hline
RSPG  & -0.633  & -0.632   & 75.1 \\
SGD with $1$ mini-batch & -0.322 & -0.322  & 96.1 \\
SGD with $10$ mini-batch & 1.36 & 1.37 & 73.4 \\
SGD with $30$ mini-batch & 2.61 & 2.61  &71.6 \\
SGD with $50$ mini-batch & 3.08 & 3.08 &   409 \\

  \multicolumn{1}{c}{} & \multicolumn{3}{c}{$N=30000$, $ p=1000$}  \\
Methods  & EmpRisk & ExpRisk   & Time \\ \hline
RSPG  & -0.65 & -0.649   & 78.4 \\
SGD with $1$ mini-batch & -0.488 & -0.487  & 148 \\
SGD with $10$ mini-batch & 0.164 & 0.165 & 79.7 \\
SGD with $30$ mini-batch & 1.34 & 1.34  & 73.9 \\
SGD with $50$ mini-batch & 1.95 & 1.95  & 576 \\

  \multicolumn{1}{c}{} & \multicolumn{3}{c}{$N=10000$, $ p=2000$}  \\
Methods  & EmpRisk & ExpRisk  & Time  \\ \hline
RSPG  & -0.647 & -0.646   & 117 \\
SGD with $1$ mini-batch & -0.131 & -0.13  & 144 \\
SGD with $10$ mini-batch & 3.23 & 3.23  & 99.1 \\
SGD with $30$ mini-batch & 5.63 & 5.63 & 95.6 \\
SGD with $50$ mini-batch & 6.52 & 6.53 &   503 \\

  \multicolumn{1}{c}{} & \multicolumn{3}{c}{$N=30000$, $ p=2000$}  \\
Methods  & EmpRisk & ExpRisk  & Time  \\ \hline
RSPG  & -0.66 & -0.66   & 125 \\
SGD with $1$ mini-batch & -0.436 & -0.435  & 250 \\
SGD with $10$ mini-batch & 0.875 & 0.875  & 112 \\
SGD with $30$ mini-batch & 3.19 & 3.19  & 100 \\
SGD with $50$ mini-batch & 4.38 & 4.38  & 675 \\

\end{tabular}
\end{table}

\begin{table}[H]
\caption{EmpRisk, ExpRisk and computation time for $\lambda=10^{-1}$   }
\label{tab3}
 \centering
 \begin{tabular}{c | c  c c }    
 \multicolumn{1}{c}{} & \multicolumn{3}{c}{$N=10000$, $p=1000$ } \\
Methods  & EmpRisk & ExpRisk  & Time   \\ \hline
RSPG  & -0.633  &  -0.632    & 74.6 \\
SGD with $1$ mini-batch & -0.411 & -0.411 & 95.6 \\
SGD with $10$ mini-batch & 0.483 & 0.483 &  72.9 \\
SGD with $30$ mini-batch & 1.53 & 1.53  &71.1 \\
SGD with $50$ mini-batch & 2.39 & 2.39  & 70.8 \\

  \multicolumn{1}{c}{} & \multicolumn{3}{c}{$N=30000$, $ p=1000$}  \\
Methods  & EmpRisk & ExpRisk   & Time \\ \hline
RSPG  & -0.64 & -0.639   & 78.1 \\
SGD with $1$ mini-batch & -0.483 & -0.482  & 148 \\
SGD with $10$ mini-batch & -4.56$\times10^{-2}$ & -4.5$\times10^{-2}$  & 79.6 \\
SGD with $30$ mini-batch & 0.563 & 0.563  & 73.7 \\
SGD with $50$ mini-batch & 0.963 & 0.963  & 238 \\

  \multicolumn{1}{c}{} & \multicolumn{3}{c}{$N=10000$, $ p=2000$}  \\
Methods  & EmpRisk & ExpRisk  & Time  \\ \hline
RSPG  & -0.654 & -0.653   & 116 \\
SGD with $1$ mini-batch & -0.462 & -0.461 &   144 \\
SGD with $10$ mini-batch & 0.671 & 0.672 &   98.9 \\
SGD with $30$ mini-batch & 2.43 & 2.44 &   95.4 \\
SGD with $50$ mini-batch & 4.02 & 4.02 &   165 \\

  \multicolumn{1}{c}{} & \multicolumn{3}{c}{$N=30000$, $ p=2000$}  \\
Methods  & EmpRisk & ExpRisk &  Time  \\ \hline
RSPG  & -0.66 & -0.66  &  130 \\
SGD with $1$ mini-batch & -0.559 & -0.558 &  262 \\
SGD with $10$ mini-batch & -9.71$\times10^{-2}$ & -9.62$\times10^{-2}$ &   116 \\
SGD with $30$ mini-batch & 0.697 & 0.697 &  104 \\
SGD with $50$ mini-batch & 1.32 & 1.32 &  340 \\

\end{tabular}
\end{table}

\section{Real data analysis}\label{7}
We applied our method `online sparse $\gamma$-Poisson' to real data `Online News Popularity' (\citet{10.1007/978-3-319-23485-4_53}), which is available at \url{https://archive.ics.uci.edu/ml/datasets/online+news+popularity}. 
%
%
%
We compared our method with sparse Poisson regression which was implemnted by R-package `glmnet' with default parameter setting.
%
%
%

Online News Popularity dataset contains 39644 samples with 58 dimensional explanatory variables.
We divided the dataset to 20000 training and 19644 test samples. 
In Online News Popularity dataset, the exposure time of each sample is different. 
Then, we used the $\log$ transformed feature value `timedelta' as the offset term. 
%
%
Moreover, 2000 training samples were randomly selected.
Outliers were incorporated into training samples as follows:
\begin{align*}
y_{outlier,i} = y_{i} + 100 \times t_{i} \quad (i=1, \ldots , 2000),
\end{align*}
where $i$ is the index of the randomly selected sample and $y_{i}$ is the response variable of the $i$-th randomly selected sample and $t_{i}$ is the offset term of the $i$-th randomly selected sample.
%
%

As a measure of predictive performance, the root trimmed mean squared prediction error (RTMSPE) was computed for the test samples given by
\begin{align*}
\textrm{RTMSPE} =\sqrt{ \frac{1}{h} \sum_{j=1}^h e_{[j]}^2 },
\end{align*}
where $ e_j^2= \left(  ( y_{j} - \left\lfloor \exp \left( \log (t_{j}) +  \hat{\beta_0} + x_j^{T} \hat{\beta} \right) \right\rfloor \right)^2  $, $e_{[1]}^2 \leq  \cdots \leq e_{[19644]}^2$ are the order statistics of $e_1^2, \cdots , e_{19644}^2$ and $h= \lfloor (19644+1)(1-\alpha) \rfloor$ with $\alpha=0.05, \cdots , 0.3$. 

In our method, we need an initial point and some tuning parameters to obtain the estimate. 
Therefore, we used $N_{init}=200$ samples which were used for estimating an initial point and other parameters $L$ in (\ref{assum1}) and $\tau^2$ in (\ref{assum2}) to calculate in advance. 
%
%
%
%
In this experiment, we used the estimate of the RANSAC. 
For estimating $L$, we followed the way to in \cite{Ghadimi:2016}, page 298-299. 
Moreover, we used the following value of tuning parameters in this experiment. 
The parameter $\gamma$ in the $\gamma$-divergence was set to $0.1, 0.5, 1.0$. 
The parameter $\lambda$ of $L_1$ regularization was selected by the robust cross-validation proposed by \citet{e19110608}. 
The robust cross-validation was given by: 
\begin{align*}
\mbox{RoCV}(\lambda) = - \frac{1}{n} \sum_{i=1}^n \frac{   f(y_i|x_i;\hat{\theta}^{[-i]})^{\gamma_0}  }{  \left( \int f(y|x_i;\hat{\theta}^{[-i]})^{1+\gamma_0}dy\right)^\frac{\gamma_0}{1+\gamma_0}}  ,
\end{align*}
where $\hat{\theta}^{[-i]}$ is the estimated parameter deleting the $i$-th observation and $\gamma_0$ is an appropriate tuning parameter. 
In this experiment, $\gamma_0$ was set to $1.0$.
The mini-batch size was set to $100, 200, 500$. 
The RSPG needed the number of candidates and post-samples $N_{cand}$ and $N_{post}$ for post-optimization as described in Algorithm \ref{general 2-RSPG}. 
We used $N_{cand}=5$ and $N_{post}= \left\lceil N/10 \right\rceil$. 
We showed the best result of our method and comparative method in Table \ref{tab4}. 
%
All results were obtained in R version 3.3.0 with Intel Core i7-4790K machine. 
Table \ref{tab4} shows that our method performed better than sparse Poisson regression. 
%

\begin{table}[H]
\caption{Root trimmed mean squared prediction error in test samples}
 \label{tab4}
 \centering
 \begin{tabular}{c | c c c c c c |}    
 \multicolumn{1}{c |}{} & \multicolumn{6}{c|}{ trimming fraction 100$\alpha \%$ } \\ \hline
Methods  & 5$\%$ & 10$\%$ & 15$\%$ & 20$\%$ & 25$\%$ & 30$\%$   \\ \hline
Our method  & 2419.3 & 1760.2 & 1423.7 & 1215.7 & 1064 & 948.9  \\
Sparse Poisson Regrssion & 2457.2  &  2118.1  &  1902.5  &  1722.9 &  1562.5 &  1414.1 \\
\end{tabular}
\end{table}

\section{Conclusions}\label{8}
We proposed the online robust and sparse GLM based on the $\gamma$-divergence. 
We applied a stochastic optimization approach in order to reduce the computational complexity and overcome the computational problem on the hypergeometric series in Poisson regression. 
We adopted the RSPG, which guaranteed the global convergence property for non-convex stochastic optimization problem, as a stochastic optimization approach. 
We proved that the global convergence property can be extended to the classical first-order necessary condition. 
In this paper, linear/logistic/Poisson regression problems with $L_1$ regularization were illustrated in detail. 
As a result, not only Poisson case but also linear/logistic case can scale well for very large problems by virtue of the stochastic optimization approach. 
To the best of our knowledge, there is no efficient method for the robust and sparse Poisson regression, but w e have succeeded to propose an efficient estimation procedure with online strategy. 
The numerical experiments and real data analysis suggested that our methods had good performances in terms of both accuracy and computational cost. 
However, there are still some problems in Poisson regression problem, e.g., overdispersion \citep{doi:10.1080/01621459.1989.10478792}, zero inflated Poisson \citep{10.2307/1269547}.
Therefore, it can be useful to extend the Poisson regression to the negative binomial regression and the zero inflated Poisson regression for future work. 
Moreover, the accelerated RSPG was proposed in \citep{Ghadimi2016}, and then we can adopt it as a stochastic optimization approach in order to achieve faster convergence than the RSPG.

\section*{Appendix}
The proof of Theorem 5.2.
\begin{align}\label{lem1}
&\lim_{ k \downarrow 0} \frac{ \Psi(\theta^{(R)} + k\delta) - \Psi(\theta^{(R)}) }{k} \nonumber \\
&= \lim_{ k \downarrow 0} \frac{  E_{(x,y)} \left[ l( (x,y);\theta^{(R)} + k\delta) \right] -  E_{(x,y)} \left[ l( (x,y);\theta^{(R)} ) \right] + \lambda P(\theta^{(R)} + k\delta) -  \lambda P(\theta^{(R)}) }{k} \nonumber \\
&= \lim_{ k \downarrow 0} \frac{  E_{(x,y)} \left[ l( (x,y);\theta^{(R)} + k\delta) \right] -  E_{(x,y)} \left[ l( (x,y);\theta^{(R)} ) \right] }{k} +  \lim_{ k \downarrow 0} \frac{   \lambda P(\theta^{(R)} + k\delta) - \lambda P(\theta^{(R)}) }{k} , 
\end{align}
The directional derivative of the differentiable function always exist and is represented by the dot product with the gradient of the differentiable function and the direction given by
\begin{align}\label{lem2}
\lim_{ k \downarrow 0} \frac{  E_{(x,y)} \left[ l( (x,y);\theta^{(R)} + k\delta) \right] -  E_{(x,y)} \left[ l( (x,y);\theta^{(R)} ) \right] }{k}  = \left\langle \nabla  E_{(x,y)} \left[ l( (x,y);\theta^{(R)}) \right] , \delta \right\rangle. 
\end{align}
Moreover, the directional derivative of the (proper) convex function exists at the relative interior point of the domain and is greater than the dot product with the subgradient of the convex function and direction \citep{rockafellar-1970a} given by
\begin{align}\label{lem3}
 \lim_{ k \downarrow 0} \frac{   \lambda P(\theta^{(R)} + k\delta) - \lambda P(\theta^{(R)}) }{k} & = \sup_{g \in \partial P(\theta^{(R)}) }   \lambda  \left\langle g , \delta \right\rangle \nonumber \\
 & \geq  \lambda  \left\langle g , \delta \right\rangle  \ for \ any \ g \in \partial P(\theta^{(R)}).
\end{align}
Then, by the optimality condition of (\ref{update true grad RSPG}), we have the following equation
\begin{align}\label{lem4}
0  & \in  \nabla E_{(x,y)} \left[ l((x,y);  \theta^{(R)} )  \right] + \lambda \partial P(\theta^{+})  +\frac{1}{\eta_R} \left\{   \nabla w \left( \theta^{+} \right) - \nabla w \left( \theta^{(R)} \right) \right\}  \nonumber \\ 
\frac{1}{\eta_R} \left\{   \nabla w \left( \theta^{(R)} \right) - \nabla w \left( \theta^{+} \right) \right\}  & \in  \nabla E_{(x,y)} \left[ l((x,y);  \theta^{(R)} ) \right]  + \lambda \partial P(\theta^{+})  .
\end{align}
Therefore, we can obtain (\ref{mod RSPG glb conv}) from $P_{X,R} \approx 0$, (\ref{lem1}),  (\ref{lem2}),  (\ref{lem3}) and  (\ref{lem4}) as follows;
\begin{align*}
& \lim_{ k \downarrow 0} \frac{  E_{(x,y)} \left[ l( (x,y);\theta^{(R)} + k\delta) \right] -  E_{(x,y)} \left[ l( (x,y);\theta^{(R)} ) \right] }{k} +  \lim_{ k \downarrow 0} \frac{   \lambda P(\theta^{(R)} + k\delta) - \lambda P(\theta^{(R)}) }{k} \\
& \geq  \left\langle \nabla  E_{(x,y)} \left[ l( (x,y);\theta^{(R)}) \right] , \delta \right\rangle +\lambda  \left\langle g , \delta \right\rangle  \qquad for \ any \ g \in \partial P(\theta^{(R)})  \\
& =\left\langle  \nabla  E_{(x,y)} \left[ l( (x,y);\theta^{(R)}) \right] +\lambda g ,\delta \right\rangle \qquad for \ any \ g \in \partial P(\theta^{(R)})  \\  
& \ni 0
\end{align*}

\bibliography{ref}

\end{document}